\documentclass{article}
\usepackage{amssymb}
\usepackage{amsmath}
\usepackage{amsthm}
\usepackage{graphicx}
\usepackage[a4paper]{geometry}
\usepackage[round]{natbib}
\usepackage[dvipsnames]{xcolor}
\usepackage[colorlinks,linktoc=true,
             citecolor= Cerulean!85!green!90!black,
             linkcolor=YellowOrange,
             urlcolor=RubineRed!85!black,
             ]{hyperref}

\usepackage{multirow}
\usepackage{enumitem}

\DeclareMathOperator*{\argmax}{arg\,max}
\DeclareMathOperator*{\argmin}{arg\,min}

\newtheorem*{theorem*}{Theorem}
\newtheorem{theorem}{Theorem}[section]

\newtheorem{lemma}[theorem]{Lemma}
\newtheorem{proposition}[theorem]{Proposition}

\newcommand{\x}{\mathbf{x}}
\newcommand{\y}{\mathbf{y}}
\newcommand{\wt}{\widehat{\theta}}
\newcommand{\fxt}{f(\x;\wt)}
\newcommand{\faxt}{f(a(\x);\wt)}
\newcommand{\fxit}{f(\x_i;\wt)}
\newcommand{\faxit}{f(a(\x_i);\wt)}
\newcommand{\ma}{\mathcal{A}}
\newcommand{\mP}{\mathcal{P}}
\newcommand{\mI}{\mathbb{I}}
\newcommand{\rpp}{r_{\mP'}}
\newcommand{\wrpp}{\widehat{r}_{\mP'}}
\newcommand{\rp}{r_\mP}
\newcommand{\wrp}{\widehat{r}_{\mP}}

\newcommand{\X}{\mathbf{X}}
\newcommand{\Y}{\mathbf{Y}}

\newcommand{\base}{\textsf{Base}}
\newcommand{\va}{\textsf{VA}}
\newcommand{\ra}{\textsf{RA}}
\newcommand{\vwa}{\textsf{VWA}}
\newcommand{\rwa}{\textsf{RWA}}

\newcommand{\raw}{\textsf{RA-W}}
\newcommand{\ral}{\textsf{RA-}$\ell_1$}
\newcommand{\rad}{\textsf{RA-D}}
\newcommand{\rak}{\textsf{RA-KL}}
\newcommand{\ras}{\textsf{RA\textsubscript{softmax}}}
\newcommand{\rasl}{\textsf{RA\textsubscript{softmax}-}$\ell_1$}

\newcommand{\indep}{\perp \!\!\! \perp}

\title{Squared $\ell_2$ Norm as Consistency Loss for Leveraging Augmented Data to Learn Robust and Invariant Representations}

\author{Haohan Wang, Zeyi Huang, Xindi Wu, and Eric P. Xing\\
\textit{School of Computer Science,}\\
\textit{Carnegie Mellon University,}\\
\textit{Pittsburgh, PA 15213, USA}\\
\textit{ \small {\{haohanw, epxing\}@cs.cmu.edu, \{zeyih, xindiw\}@andrew.cmu.edu}}
}

\date{}

\begin{document}

\maketitle

\begin{abstract}
Data augmentation is one of the most popular techniques 
for improving the robustness of neural networks.
In addition to directly training the model with original samples and augmented samples, 
a torrent of methods regularizing the distance between
embeddings/representations of the original samples 
and their augmented counterparts have been introduced.
In this paper, we explore these various regularization choices, 
seeking to provide a general understanding of 
how we should regularize the embeddings.
Our analysis suggests the ideal choices of regularization
correspond to various assumptions. 
With an invariance test, we argue that regularization is  
important if the model is to be used in a broader context 
than accuracy-driven setting
because non-regularized approaches are limited in 
learning the concept of invariance, despite equally high accuracy. 
Finally, we also show that the generic approach
we identified (squared $\ell_2$ norm regularized augmentation)
outperforms several recent methods,
which are each specially designed for one task 
and significantly more complicated than ours, 
over three different tasks. 
\end{abstract}

\section{Introduction}
Recent advances in deep learning has delivered 
remarkable empirical performance over \textit{i.i.d} test data, 
and the community continues to investigate the more challenging 
and realistic scenario 
when models are tested in robustness 
over non-\textit{i.i.d} data 
\citep[\textit{e.g.},][]{ben2010theory,szegedy2013intriguing}.  
Recent studies
suggest that one challenge is 
the model's tendency in capturing undesired signals \citep{geirhos2018imagenettrained,wang2020high}, 
thus combating this tendency may be a key to robust models. 

To help models ignore the undesired signals, data augmentation 
(\textit{i.e.}, diluting the undesired signals of training samples
by applying transformations to existing examples) 
is often used.
Given its widely usage, 
we seek to answer the question:
\emph{how should we train with augmented samples 
so that the assistance of augmentation 
can be taken to the fullest extent to learn robust and invariant models?}

In this paper, 
We analyze the generalization behaviors 
of models trained with augmented data
and associated regularization techniques. 
We investigate a set of assumptions
and compare the worst-case expected risk over unseen data 
when \textit{i.i.d} samples are allowed to be transformed 
according to a function belonging to a family. 
We bound the expected risk with terms 
that can be computed during training, 
so that our analysis can inspire 
how to regularize the training procedure. 
While all the derived methods have an upper bound of the expected risk, 
with progressively stronger assumptions, 
we have progressively simpler regularization, 
allowing practical choices to be made according to 
the understanding of the application.
Our contributions of this paper are as follows: 
\begin{itemize}
    \item We offer analyses of the generalization behaviors of  augmented models trained with different regularizations: these regularizations require progressively stronger assumptions of the data and the augmentation functions, but progressively less computational efforts. For example, with assumptions pertaining to augmentation transformation functions, the Wasserstein distance over the original and augmented empirical distributions can be calculated through simple $\ell_1$ norm distance. 
    \item We test and compare these methods and offer practical guidance on how to choose regularizations in practice. 
    In short, regularizing the squared $\ell_2$ distance of logits 
    between the augmented samples and original samples is a favorable method, 
    suggested by both theoretical and empirical evidence. 
    \item With an invariance test, 
    we argue that vanilla augmentation does not 
    utilize the augmented samples to the fullest extent, especially in learning invariant representations, 
    thus may not be ideal unless the only goal of augmentation is 
    to improve the accuracy over a specific setting. 
\end{itemize}


\section{Related Work \& Key Differences}
\label{sec:related}

Data augmentation has been used 
effectively for years.
Tracing back to the earliest convolutional neural networks, 
we notice that even the LeNet applied on MNIST dataset 
has been boosted by mixing the distorted images to the 
original ones \citep{lecun1998gradient}. 
Later, the rapidly growing machine learning community 
has seen a proliferate development of data augmentation techniques 
(\textit{e.g.}, flipping, rotation, blurring \textit{etc.})
that have helped models climb the ladder of the state-of-the-art (one may refer to relevant survey \citep{shorten2019survey} for details). 
Recent advances expanded the conventional concept of data augmentation and invented several new approaches, 
such as leveraging the information in unlabelled data \citep{xie2019unsupervised}, 
automatically learning augmentation functions \citep{ho2019population,hu2019learning,wang2019implicit,zhang2020adversarial,zoph2019learning}, 
and generating the samples (with constraint) 
that maximize the training loss along training \citep{FawziSTF16}, 
which is later widely accepted as adversarial training \citep{MadryMSTV18}. 

While the above works mainly discuss 
how to generate the augmented samples, 
in this paper, we mainly answer the question about 
how to train the models with augmented samples. 
For example, instead of directly mixing augmented samples 
with the original samples, 
one can consider regularizing the representations (or outputs)
of original samples and augmented samples 
to be close under a distance metric (also known as a consistency loss). 
Many concrete ideas have been explored in different contexts. 
For example, $\ell_2$ distance and cosine similarities
between internal representations in speech recognition
\citep{liang2018learning}, 
squared $\ell_2$ distance between logits \citep{kannan2018adversarial},
or KL divergence between softmax outputs \citep{ZhangYJXGJ19} in adversarially robust vision models, 
Jensen–Shannon divergence (of three distributions) 
between embeddings for texture invariant image classification \citep{Hendrycks2020augmix}. 
These are but a few highlights of the concrete and successful implementations for different applications
out of a huge collection (\textit{e.g.}, \citep{WuMWZGXX19,ZFY019, zhang2019regularizing, Shah_2019_CVPR, asai2020logicguided,sajjadi2016regularization,zheng2016improving,xie2015hyper}), 
and one can easily imagine methods permuting these three elements (distance metrics, representation or outputs, and applications) 
to be invented. 
Even further, although we are not aware
of the following methods in the context of data augmentation, 
given the popularity of GAN \citep{goodfellow2016nips} 
and domain adversarial neural network \citep{ganin2016domain}, 
one can also expect the distance metric generalizes 
to a specialized discriminator (\textit{i.e.} a classifier), 
which can be intuitively understood as a calculated
(usually maximized) distance measure, 
Wasserstein-1 metric as an example \citep{arjovsky2017wasserstein,gulrajani2017improved}. 

\textbf{Key Differences:}
With this rich collection of regularizing choices, 
which one method should we consider in general? 
More importantly, 
do we actually need the regularization at all? 
These questions are important 
for multiple reasons,
especially considering that 
there are paper suggesting that 
these regularizations may lead to worse results \citep{jeong2019consistency}. 
In this paper, we answer the first question 
with a proved upper bound 
of the worst case generalization error, 
and our upper bound explicitly describes what regularizations are needed. 
For the second question, 
we will show that regularizations can 
help the model to learn the concept of invariance. 

There are also several previous discussions 
regarding the detailed understandings of data augmentation \citep{yang2019invariance, chen2019grouptheoretic,hernndezgarca2018data,rajput2019does,DaoGRSSR19}, 
among which, \citep{yang2019invariance} is probably the most relevant 
as it also defends the usage of regularizations. 
However, we believe our discussions are more comprehensive
and supported theoretically, 
since our analysis directly suggests the ideal regularization.
Also, empirically, we design an invariance test 
in addition to the worst-case accuracy used in the preceding work. 

\section{Training Strategies with Augmented Data}
\label{sec:method}

\paragraph{Notations}
$(\X, \Y)$ denotes the data,
where $\X \in \mathcal{R}^{n\times p}$ and $\Y \in \{0, 1\}^{n\times k}$
(one-hot vectors for $k$ classes), 
and $f(\cdot, \theta)$ denotes the model,
which takes in the data and outputs the softmax (probabilities of the prediction) 
and $\theta$ denotes the corresponding parameters. 
$g()$ completes the prediction (\textit{i.e.},
mapping softmax to one-hot prediction). 
$l(\cdot, \cdot)$ denotes a generic loss function. 
$a(\cdot)$ denotes a transformation that alters the undesired signals of a sample, 
\textit{i.e.}, the data augmentation method. $a\in \ma$, 
which is the set of transformation functions.
$\mP$ denotes the distribution of $(\x, \y)$.
For any sampled $(\x, \y)$, we can have $(a(\x), \y)$, 
and we use $\mP_a$ to denote the distribution of these transformed samples. 
$r(\cdot;\theta)$ denotes the risk of model $\theta$. 
$\widehat{\cdot}$ denotes the estimated term $\cdot$.

\subsection{Well-behaved Data Transformation Function}

Despite the strong empirical performance data augmentation has demonstrated, 
it should be intuitively expected that the performance 
can only be improved when the augmentation is chosen wisely. 
Therefore, before we proceed to analyze the behaviors of 
training with data augmentations, 
we need first regulate some basic properties of the data transformation functions used. 
Intuitively, we will consider the following three properties. 
\begin{itemize}
    \item ``Dependence-preservation'' with two perspectives:  
    Label-wise, the transformation cannot alter the label of the data, 
    which is a central requirement of almost all the data augmentation practice. 
    Feature-wise, the transformation will not introduce new dependencies between the samples. 
    Notice that this dependence-preservation assumption appears strong, but it is one of the central assumptions required to derive the generalization bounds. 
    \item ``Efficiency'': the augmentation should only generate new samples 
    of the same label as minor perturbations of the original one.
    If a transformation violates this property,
    there should exist other simpler transformations
    that can generate the same target sample. 
    \item ``Vertices'': There are extreme cases of the transformations. 
    For example, if one needs the model to be invariant to rotations from $0^{\circ}$ to $60^{\circ}$,
    we consider the vertices to be $0^{\circ}$ rotation function (thus identity map) and $60^{\circ}$ rotation function. 
    In practice, one usually selects the transformation vertices with intuitions and domain knowledge. 
\end{itemize}

We now formally define these three properties. 
The definition will depend on the model, 
thus these properties are not only regulating 
the transformation functions, but also the model. 
We introduce the Assumptions A1-A3 corresponding to the properties. 
\begin{itemize}
    \item[\textbf{A1}:] Dependence-preservation: 
    the transformation function will not alter the dependency regarding the label (\textit{i.e.}, for any $a()\in \ma$, $a(\x)$ will have the same label as $\x$)
    or the features (\textit{i.e.}, for any $a_1(), a_2() \in \ma$, $a_1(\x_1) \indep a_2(\x_2)$ for any $\x_1, \x_2 \in \X$ that $\x_1\neq \x_2$, in other words, $a_1(\x_1)$ and $a_2(\x_2)$ are independent if $\x_1$ and $\x_2$ are independent).
    \item[\textbf{A2}:] Efficiency: for $\wt$ and 
    any $a()\in \ma$, $\faxt$ is closer to $\x$ than any other samples under a distance metric 
    $d_e(\cdot, \cdot)$, \textit{i.e.}, $d_e(\faxt, \fxt)\leq\min_{\x'\in\X_{-\x}}d_e(\faxt, f(\x';\wt))$.
    \item[\textbf{A3}:] Vertices:
    For a model $\wt$ and a transformation $a()$, we use $\mP_{a, \wt}$ to denote the distribution of $\faxt$ for $(\x, \y)\sim \mP$. 
    ``Vertices'' argues that exists two extreme elements in $\ma$, namely $a^+$ and $a^-$, with certain metric $d_x(\cdot, \cdot)$, we have
    \begin{align}
        d_x (\mP_{a^+, \wt}, \mP_{a^-, \wt}) = \sup_{a_1, a_2\in \ma} d_x (\mP_{a_1, \wt}, \mP_{a_2, \wt})
    \end{align}
\end{itemize}
Note that $d_x(\cdot,\cdot)$ is a metric over two distributions and 
$d_e(\cdot,\cdot)$ is a metric over two samples. 
Also, slightly different from the intuitive understanding of ``vertices'' above, 
\textbf{A3} regulates the behavior of embedding instead of raw data. 
All of our follow-up analysis will require \textbf{A1} to hold, 
but with more assumptions held,
we can get computationally lighter methods with bounded error.

\subsection{Background, Robustness, and Invariance}
One central goal of machine learning 
is to understand the generalization error. 
When the test data and train data are from the same distribution, 
many previous analyses can be sketched as:
\begin{align}
    \rp(\wt) \leq \wrp(\wt) + \phi(|\Theta|, n, \delta)
    \label{eq:standard}
\end{align}
which states that the expected risk can be bounded by the empirical risk 
and a function of hypothesis space $|\Theta|$ and number of samples $n$; 
$\delta$ accounts for the probability when the bound holds. 
$\phi()$ is a function of these three terms. 
Dependent on the details of different analyses, 
different concrete examples of this generic term will need different assumptions. 
We use a generic assumption \textbf{A4} to denote 
the assumptions required for each example. 
More concrete discussions are in Appendix~\ref{sec:app:assumption}

\paragraph{Robustness}
In addition to the generalization error above, 
we also study the robustness by following the established definition 
as in the worst case expected risk
when the test data is allowed to 
be shifted to some other distributions 
by transformation functions in $\ma$.
Formally, we study
\begin{align}
    \rpp(\wt) =  \mathbb{E}_{(\x, \y) \sim \mP} \max_{a\sim \ma}\mI(g(\faxt) \neq \y) 
    \label{eq:main}
\end{align}
As $\rp(\wt) \leq \rpp(\wt)$, we only need to study \eqref{eq:main}. 
We will analyze \eqref{eq:main} in different scenarios involving different assumptions 
and offer formalizations of the generalization bounds under each scenario.
Our bounds shall also immediately inspire the development of methods 
in each scenario as the terms involved in our bound 
are all computable within reasonable computational loads. 

\paragraph{Invariance}
In addition to robustness, 
we are also interested in whether the model 
learns to be invariant to the undesired signals.  
Intuitively, 
if data augmentation is used to help 
dilute the undesired signals from data
by altering the undesired signals with $a() \in \ma$, 
a successfully trained model 
with augmented data 
will map the raw data with various undesired signals 
to the same embedding. 
Thus, 
we study the following metric 
to quantify the model's ability in learning invariant representations:
\begin{align}
    I(\wt, \mP) = \sup_{a_1, a_2\in \ma} d_x (\mP_{a_1, \wt}, \mP_{a_2, \wt}), 
    \label{eq:invariance}
\end{align}
where 
$\mP_{a, \wt}$ to denote the distribution of $\faxt$ for $(\x, \y)\sim \mP$. 
$d_x()$ is a distance over two distributions, and we suggest to use 
Wasserstein metric given its  favorable properties 
(\textit{e.g.}, see practical examples in Figure 1 of
\citep{cuturi2014fast} or theoretical discussions in \citep{villani2008optimal}). 
Due to the difficulties in assessing $\faxt$
(as it depends on $\wt$), 
we mainly study \eqref{eq:invariance} empirically, 
and argue that 
models trained with explicit regularization 
of the empirical counterpart
of \eqref{eq:invariance} will have favorable invariance property. 


\subsection{Worst-case Augmentation (Adversarial Training)}

We consider robustness first. \eqref{eq:main} can be written equivalently into the expected risk 
over a pseudo distribution $\mP'$ (see Lemma 1 in \citep{tu2019theoretical}), 
which is the distribution that can sample the data leading to the  worst expected risk. 
Thus, equivalently, we can consider 
$\sup_{\mathcal{P}'\in T(\mP, \ma)} \rpp(\wt)$. 
With an assumption relating the worst distribution of expected risk 
and the worst distribution of the empirical risk 
(namely, \textbf{A5}, in  Appendix~\ref{sec:app:assumption}),
the bound of our interest (\textit{i.e.}, $\sup_{\mathcal{P'}\in T(\mP, \ma)} \rpp(\wt)$) 
can be analogously analyzed through $\sup_{\mathcal{P'}\in T(\mP, \ma)} \wrpp(\wt)$. 
By the definition of $\mP'$, we can have:
\begin{lemma}
With Assumptions A1, A4, and A5, with probability at least $1-\delta$, we have 
\begin{align}
    \sup_{\mathcal{P'}\in T(\mP, \ma)} 
    \rpp(\wt)  \leq 
    \dfrac{1}{n}\sum_{(\x, \y) \sim \mP}\sup_{a\in \ma}\mI(g(\faxt) \neq \y)  + \phi(|\Theta|, n, \delta)
\end{align}
\end{lemma}

This result is a straightforward follow-up of the preceding discussions. 
In practice, it aligns with the adversarial training \citep{MadryMSTV18}, 
a method that has demonstrated impressive empirical successes in the robust machine learning community. 

While the adversarial training has been valued by its empirical superiorities,
it may still have the following two directions that can be improved: 
firstly, it lacks an explicit enforcement of the concept of invariance 
between the original sample and the transformed sample; 
secondly, it assumes that elements of $\ma$ are enumerable, 
thus $\dfrac{1}{n}\sum_{(\x, \y) \sim \mP}\sup_{a\in \ma}\mI(g(\faxt) \neq \y)$ is computable. 
The remaining discussions expand along these two directions. 

\subsection{Regularized Worst-case Augmentation}
To force the concept of invariance, 
the immediate solution might be to apply some regularizations 
to minimize the distance between the embeddings 
learned from the original sample and the ones learned from the transformed samples. 
We have offered a summary of these methods in Section~\ref{sec:related}. 


To have a model 
with small invariance score, 
the direct approach will be regularizing 
the empirical counterpart of \eqref{eq:invariance}. 
We notice that existing methods barely consider 
this regularization, 
probably because of the 
computational difficulty of Wasserstein distance. 
Conveniently, we have the following result
that links the $\ell_1$ regularization 
to the Wasserstein-1 metric in the context of data augmentation. 
\begin{proposition}
With A2, and $d_e(\cdot,\cdot)$ in A2 chosen to be $\ell_1$ norm, for any $a\in \ma$, we have
\begin{align}
    \sum_{i} ||\fxit-\faxit||_1 = W_1(\fxt, \faxt)
\end{align}
\end{proposition}

This result conveniently allows us to use $\ell_1$ norm distance to replace Wasserstein metric, 
integrating the advantages of Wasserstein metric while avoiding practical issues
such as computational complexity and difficulty to pass the gradient back during backpropagation. 

We continue to discuss the generalization behaviors. 
Our analysis remains in the scope of multi-class classification,
where the risk is evaluated as misclassification rate, 
and the model is optimized with cross-entropy loss 
(with the base chosen to be log base in cross-entropy loss). 
This setup aligns with \textbf{A4}, 
and should represent the modern neural network studies well enough. 

Before we proceed, we need another technical assumption \textbf{A6}
(details in Appendix~\ref{sec:app:assumption}), 
which can be intuitively considered as a tool that allows us 
to relax classification error into cross-entropy error, 
so that we can bound the generalization error 
with the terms we can directly optimize during training. 

We can now offer another technical result:
\begin{theorem}
With Assumptions A1, A2, A4, A5, and A6, and $d_e(\cdot,\cdot)$ in A2 is $\ell_1$ norm, with probability at least $1-\delta$, the worst case generalization risk will be bounded as
\begin{align}
    \sup_{\mathcal{P'}\in T(\mP, \ma)} 
    \rpp (\wt)  \leq 
    \wrp (\wt) + 
    \sum_{i}||f(\x_i;\wt) - f(\x'_i;\wt)||_1 + 
    \phi(|\Theta|, n, \delta)
\end{align}
and 
$\x' = a(\x) $, where $ a = \argmin_{a \in \ma} \y^\top \faxt$.
\end{theorem}

This technical result also immediately inspires the method to guarantee worst case performance, 
as well as to explicitly enforce the concept of invariance. 
Notice that $ a = \argmax_{a \in \ma} \y^\top \faxt$ is simply selecting the augmentation function 
maximizing the cross-entropy loss, 
a standard used by many worst case
augmenting method \citep[\textit{e.g.},][]{MadryMSTV18}. 

\subsection{Regularized Training with Vertices}

As $\ma$ in practice is usually a set with a large number of
(and possibly infinite) elements, we may not always be able to identify 
the worst case transformation function with reasonable computational efforts. 
This limitation also prevents us from effective estimating the generalization error
as the bound requires the identification of the worst case transformation. 

Our final discussion is to leverage the vertex property of the transformation function 
to bound the worst case generalization error:
\begin{lemma}
With Assumptions A1-A6, and $d_e(\cdot, \cdot)$ in A2 chosen as $\ell_1$ norm distance, 
$d_x(\cdot, \cdot)$ in A3 chosen as Wasserstein-1 metric,
assuming there is a $a'()\in \ma$ where $ \widehat{r}_{\mP_{a'}}(\wt)=\frac{1}{2}\big(\widehat{r}_{\mP_{a^+}}(\wt)+\widehat{r}_{\mP_{a^-}}(\wt)\big)$,
with probability at least $1-\delta$, we have:
\begin{align*}
    \sup_{\mathcal{P'}\in T(\mP, \ma)} 
    \rpp(\wt)  \leq &
    \dfrac{1}{2}\big(\widehat{r}_{\mP_{a^+}}(\wt) + \widehat{r}_{\mP_{a^-}}(\wt)\big) + 
    \sum_{i}||f(a^+(\x_i);\wt) - f(a^-(\x');\wt)||_1 + 
    \phi(|\Theta|, n, \delta)
\end{align*}
\end{lemma}

This result inspires the method that can directly guarantee 
the worst case generalization result and can be optimized conveniently 
without searching for the worst-case transformations. 
However, this method requires a good domain knowledge of the vertices of the transformation functions. 

\subsection{Engineering Specification of Relevant Methods}
Our theoretical analysis has lead to a line of methods, 
however, not every method can be effectively implemented,
especially due to the difficulties of 
passing gradient back for optimizations. 
Therefore, to boost the influence of the loss function through backpropagation, 
we recommend to adapt the methods with the following two changes: 
1) the regularization is enforced on logits instead of softmax; 
2) we use squared $\ell_2$ norm instead of $\ell_1$ norm 
because $\ell_1$ norm is not differentiable everywhere. 
We discuss the effects of these compromises in ablation studies in  Appendix~\ref{sec:app:more}. 

Also, in the cases where we need to identify the worst case transformation functions,
we iterate through all the transformation functions and identify the function with the maximum loss. 

Overall, our analysis leads to the following main training strategies: 
\begin{itemize}
    \item \va{} (vanilla augmentation): mix the augmented samples of a vertex function to the original ones for training (original samples are considered as from another vertex in following experiments). 
    \item \vwa{} (vanilla worst-case augmentation): at each iteration, identify the worst-case transformation functions and train with samples generated by them (also known as adversarial training). 
    \item \ra{} (regularized augmentation): regularizing the squared $\ell_2$ distance over logits between the original samples and the augmented samples of a fixed vertex transformation function. 
    \item \rwa{} (regularized worst-case augmentation): regularizing the squared $\ell_2$ distance over logits between the original samples and the worst-case augmented samples identified at each iteration. 
\end{itemize}

\section{Experiments}
\label{sec:exp}
We first use some synthetic experiments 
to verify our assumptions and inspect 
the consequences when the assumptions are not met (in Appendix~\ref{sec:app:assump:validation}). 
Then, in the following paragraphs, 
we test the methods discussed 
to support our arguments in learning robustness and invariance. 
Finally, we show the power of our discussions
by competing with advanced methods 
designed for specific tasks.

\subsection{Experiments for Learning Robust \& Invariant Representation}
\label{sec:exp:sync}

\paragraph{Experiment Setup:}
We first test our arguments 
with two data sets and 
three different sets of the augmentations.
We study MNIST dataset with LeNet architecture, 
and CIFAR10 dataset with ResNet18 architecture. 
To examine the effects of the augmentation strategies, 
we disable all the heuristics that are frequently used 
to boost the test accuracy of models, 
such as the default augmentation many models trained for CIFAR10 adopted, 
and the BatchNorm (also due to the recent arguments against 
the effects of BatchNorm in learning robust features \citep{wang2020high}),  
although forgoing these heuristics will result in 
a lower overall performance than one usually expects. 

We consider three different sets of transformation functions: 
\begin{itemize}
    \item Texture: we use Fourier transform to perturb the texture of the data by discarding the high-frequency components of the given a radius $r$. The smaller $r$ is, the less high-frequency components the image has. We consider $\ma=\{a(), a_{12}(), a_{10}(), a_{8}(), a_{6}()\}$, where $a()$ is the identity map. Thus, vertexes are $a()$ and $a_6()$. 
    \item Rotation: we rotate the images clockwise $r$ degrees. We consider $\ma=\{a(), a_{15}(), a_{30}(), $ $a_{45}(), a_{60}()\}$, where $a()$ is the identity map. Thus, vertexes are $a()$ and $a_{60}()$.
    \item Contrast: we create the images depicting the same semantic information, but with different scales of the pixels, including the negative color representation. 
    Therefore, we have $\ma = \{a(\x)=\x, a_1(\x) = \x/2, a_2(\x) = \x/4, a_3(\x) = 1-\x, a_4(\x) = (1-\x)/2, a_5(\x) = (1-\x)/4$, where $\x$ stands for the image whose pixel values have been normalized to be between 0 and 1. 
    We consider $a()$ and $a_3()$ as vertexes. 
\end{itemize}

We first train the baseline models to get reasonably high performance, and then train other augmented models with the same hyperparameters. 
\va{} and \ra{} are augmented with vertexes, 
while \vwa{} and \rwa{} are augmented with $\ma$. 
For methods with a regularizer, we run the experiments with 
9 hyperparameters evenly split in the logspace from $10^{-4}$ to $10^4$, and we report the methods with the best worst-case accuracy.

\paragraph{Evaluation Metrics:} We consider three different evaluation metrics: 
\begin{itemize}
    \item Clean: test accuracy on the original test data, mainly reported as a reference for other metrics. 
    \item Robustness: the worst accuracy when each sample can be transformed with $a \in \ma$.  
    \item Invariance: A metric to test whether the models learns the concept of invariance (details to follow). 
\end{itemize}

\textbf{Invariance-test:} To test whether a model can truly learns 
the concept of invariance within $\ma = \{a_1(), a_2(), \dots, a_t()\}$ of $t$ elements, 
we design a new evaluation metric:
for a sampled collection of data of the sample label $i$, 
denoted as $\X^{(i)}$, we generate the transformed copies of it with $\ma$, 
resulting in $\X^{(i)}_{a_1}, \X^{(i)}_{a_2}, \dots, \X^{(i)}_{a_t}$. 
We combined these copies into a dataset, denoted as $\mathcal{X}^{(i)}$. 
For every sample $\x$ in $\mathcal{X}^{(i)}$, 
we retrieve its $t$ nearest neighbors of other samples in $\mathcal{X}^{(i)}$, 
and calculate the overlap of the retrieved samples 
and $\{a_1(\x), a_2(\x), \dots, a_t(\x)\}$. 
Since the identify map is in $\ma$, 
so the calculated overlap score will be in $[1/t, 1]$. 
The distance used is
$d(\cdot, \cdot) = ||f(\cdot;\wt) - f(\cdot;\wt)||_1$, 
where $\wt$ is the model we are interested to examine. 
Finally, we report the averaged score for every label. 
Thus, a high overlap score 
indicates the prediction of model $\wt$ is invariant to 
the augmentation functions in $\ma$. 
If we use other distance functions, 
the reported values may differ, but we notice that
the rank of the methods compared in terms of 
this test barely changes. 

\begin{table}[]
\small 
\centering 
\begin{tabular}{c|ccc|ccc|ccc}
\hline
 & \multicolumn{3}{c|}{Texture} & \multicolumn{3}{c|}{Rotation} & \multicolumn{3}{c}{Pixel-value} \\
 & C & R & I & C. & R. & I & C & R & I \\ \hline
\base{} & 0.9920 & 0.9833 & 0.9236 & 0.9920 & 0.2890 & 0.2059 & 0.9920 & 0.2595 & 0.2074 \\
\va{} & \textbf{0.9923} & 0.9902 & 0.9916 & 0.9899 & 0.9364 & 0.5832 & 0.9890 & 0.9543 & 0.3752 \\
\ra{} & 0.9908 & 0.9905 & \textbf{1.0000} & 0.9930 & 0.9526 & 0.6540 & 0.9937 & 0.9738 & 0.4128 \\
\vwa{} & 0.9919 & 0.9900 & 0.9976 & 0.9466 & 0.9403 & 0.6427 & 0.8400 & 0.8080 & 0.3782 \\
\rwa{} & 0.9911 & \textbf{0.9909} & \textbf{1.0000} & \textbf{0.9935} & \textbf{0.9882} & \textbf{0.9293} & \textbf{0.9938} & \textbf{0.9893} & \textbf{0.8894} \\ \hline
\end{tabular}
\caption{Results of MNIST data (``C'' stands for clean accuracy, 
``R'' stands for robustness, and ``I'' stands for invariance score): 
invariance score shows big differences while accuracy does not.}
\label{tab:main:mnist}
\end{table}

\textbf{Results:}
We show the results in Table~\ref{tab:main:mnist} and Table~\ref{tab:main:cifar} (in Appendix) for MNIST and CIFAR10 respectively. 
Table~\ref{tab:main:mnist} shows that \rwa{} is generally a superior method, in terms of all the metrics, especially the invariance evaluation as 
it shows a much higher invariance score than competing methods. 
We believe this advantage of invariance 
comes from two sources: 
regularizations and
the fact that \rwa{} has seen all the augmentation functions in $\ma$. 
In comparison, \ra{} also has regularization but only sees the vertices in $\ma$, so the invariance score of \ra{} is not compatitable to \rwa{}, although better than \va{}. 
Table~\ref{tab:main:cifar} roughly tells the same story. More discussions are in Appendix~\ref{sec:app:synthetic}.

\textbf{Other results (Appendix \ref{sec:app:more}):}
The strength of \rwa{} can also be shown in several other different scenarios, 
even in the out-of-domain test scenario where the transformation functions are not in $\ma$. 
\rwa{} generally performs the best, although not the best in every single test. 
We also perform ablation test to validate the choice of squared $\ell_2$ norm over logits in contrast to other distance metrics. 
Our choice performs the best in the worst-case performance. 
This advantage is expected as our choice is validated by theoretical arguments as well as consideration of engineering convenience. 


Overall, the empirical performances align with our expectation from the theoretical discussion:
while all methods discussed have a bounded worst case performance, 
we do not intend to compare the upper bounds 
because smaller upper bounds do not necessarily guarantee a smaller risk. 
However, 
worst case augmentation methods tend to show a better
worst case performances because they have been augmented with all the elements in $\ma$. 
Also, there is no clear evidence suggesting 
the difference between augmentation methods and its regularized 
versions in terms of the worst case performance, 
but it is clear that 
regularization helps to learn the concept of invariance.

\subsection{Comparison to Advanced Methods}
Finally, we also compete our generic data augmentation methods against 
several specifically designed methods in different applications. 
We use the four generic methods (\va{}, \ra{}, \vwa{}, and \rwa{}) with generic transformation functions 
($\ma$ of ``rotation'', ``contrast'', or ''texture'' used in the synthetic experiments).
We compare our methods with 
techniques invented for three different topics of study
(rotation invariant, texture perturbation, and cross-domain generalization), 
and each of these topics 
has seen a long line of method development.  
We follow each own tradition 
(\textit{e.g.}, rotation methods are usually 
tested in CIFAR10 dataset, seemingly due to the methods' computational requirements), 
test 
over each own most challenging dataset
(\textit{e.g.}, ImageNet-Sketch is the most recent 
and challenging
dataset in domain generalization, although less studied), 
and report each own evaluation metric
(\textit{e.g.}, methods tested with ImageNet-C 
are usually evaluated with mCE). 

Overall, the performances of our generic methods 
outperform 
these advanced SOTA techniques. 
Thus, the main conclusion, 
as validated by these challenging scenarios, 
are
(1) usage of data augmentation 
can outperform carefully designed methods;
(2) usage of the consistency loss 
can further improve the performances;
(3) regularized worst-case augmentation generally works the best. 


\begin{table}[t]
\small
\centering 
\begin{tabular}{ccccccccccc}
\hline
 & 300 & 315 & 330 & 345 & 0 & 15 & 30 & 45 & 60 & avg. \\ \hline
\base{} & 0.2196 & 0.2573 & 0.3873 & 0.6502 & 0.8360 & 0.6938 & 0.4557 & 0.3281 & 0.2578 & 0.4539 \\
\textsf{ST} & 0.2391 & 0.2748 & 0.4214 & 0.7049 & 0.8251 & 0.7147 & 0.4398 & 0.2838 & 0.2300 & 0.4593 \\
\textsf{GC} & 0.1540 & 0.1891 & 0.2460 & 0.3919 & 0.5859 & 0.4145 & 0.2534 & 0.1827 & 0.1507 & 0.2853 \\
\textsf{ETN} & 0.3855 & \textbf{0.4844} & \textbf{0.6324} & \textbf{0.7576} & 0.8276 & 0.7730 & 0.7324 & 0.6245 & 0.5060 & 0.6358 \\
\va{} & 0.2233 & 0.2832 & 0.4318 & 0.6364 & 0.8124 & 0.6926 & 0.5973 & 0.7152 & 0.7923 & 0.5761 \\
\ra{} & 0.3198 & 0.3901 & 0.5489 & 0.7170 & 0.8487 & 0.7904 & 0.7455 & 0.8005 & 0.8282 & 0.6655 \\
\vwa{} & 0.3383 & 0.3484 & 0.3835 & 0.4569 & 0.7474 & 0.866 & 0.8776 & 0.8738 & 0.8629 & 0.6394 \\
\rwa{} & \textbf{0.4012} & 0.4251 & 0.4852 & 0.6765 & \textbf{0.8708} & \textbf{0.8871} & \textbf{0.8869} & \textbf{0.8870} & \textbf{0.8818} & \textbf{0.7113} \\ \hline
\end{tabular}
\caption{Comparison to advanced rotation-invariant models. We report the test accuracy on the test sets clockwise rotated, $0^{\circ}$-$60^{\circ}$ and $300^{\circ}$-$360^{\circ}$.
Average accuracy is also reported. 
Augmentation methods only consider $0^{\circ}$-$60^{\circ}$ clockwise rotations during training. }
\label{tab:real:rotation}
\end{table}

\begin{table}[t]
\small
\centering
\begin{tabular}{cccccccccccc}
\hline
 & \base{} & \textsf{SU} & \textsf{AA} & \textsf{MBP} & \textsf{SIN} & \textsf{AM} & \textsf{AMS} & \va{} & \ra{} & \vwa{} & \rwa{} \\ \hline
Clean & 23.9 & 24.5 & 22.8 & 23 & 27.2 & \textbf{22.4} & 25.2 & 23.7 & 23.6 & 23.3 & \textbf{22.4} \\
mCE & 80.6 & 74.3 & 72.7 & 73.4 & 73.3 & 68.4 & 64.9 & 76.3 & 75.6 & 74.8 & \textbf{64.6} \\ \hline
\end{tabular}
\caption{Summary comparison to advanced models over ImageNet-C data. Performance reported (mCE) follows the standard in ImageNet-C data: clean error and mCE are both the smaller the better. 
}
\label{tab:real:c}
\end{table}

\begin{table}[t]
\small
\centering 
\begin{tabular}{ccccccccc}
\hline
 & \base{} & \textsf{InfoDrop} & \textsf{HEX} & \textsf{PAR} & \va{} & \ra{} & \vwa{} & \rwa{} \\ \hline
Top-1 & 0.1204 & 0.1224 & 0.1292 & 0.1306 & 0.1362& 0.1405& 0.1432 & \textbf{0.1486} \\
Top-5 & 0.2408 & 0.256 & 0.2564 & 0.2627 &0.2715 &0.2793 & 0.2846 & \textbf{0.2933} \\ \hline
\end{tabular}
\caption{Comparison to advanced cross-domain image classification models, over ImageNet-Sketch dataset. We report top-1 and top-5 accuracy following standards on ImageNet related experiments. 
}
\label{tab:real:sketch}
\end{table}

\paragraph{Rotation-invariant Image Classification}
We compare our results with specifically designed rotation-invariant models,
mainly Spatial Transformer (\textsf{ST}) \citep{jaderberg2015spatial},
Group Convolution (\textsf{GC}) \citep{cohen2016group}, 
and Equivariant Transformer Network (\textsf{ETN}) \citep{tai2019equivariant}. 
We also attempted to run CGNet \citep{kondor2018clebsch}, 
but the procedure does not scale to the CIFAR10 and ResNet level. 
The results are reported in Table~\ref{tab:real:rotation}, 
where most methods use the same architecture 
(ResNet34 with most performance boosting heuristics enabled), 
except that \textsf{GC} uses ResNet18 because ResNet34 with \textsf{GC} 
runs 100 times slowly than others, thus not practical. 
We test the models with nine different rotations including $0^{\circ}$ degree rotation. 
Augmentation related methods are using the $\ma$ of ``rotation''
in synthetic experiments), 
so the testing scenario goes beyond what the augmentation methods have seen during training. 
The results in Table~\ref{tab:real:rotation} strongly endorses the 
efficacy of augmentation-based methods. 
Interestingly, regularized augmentation methods,
with the benefit of learning the concept of invariance, 
tend to behave well in the transformations not considered during training. 
As we can see, \ra{} outperforms \vwa{} on average. 

\paragraph{Texture-perturbed ImageNet classification}
We also test the performance on the image classification over multiple perturbations. 
We train the model over standard ImageNet training set and test the model with ImageNet-C data \citep{hendrycks2019robustness}, which is a perturbed version of ImageNet by corrupting the original ImageNet validation set with a collection of noises. Following the standard, the reported performance is mCE, which is the smaller the better. 
We compare with several methods tested on this dataset, including 
Patch Uniform (\textsf{PU}) \citep{lopes2019improving}, 
AutoAugment (\textsf{AA}) \citep{cubuk2019autoaugment}, 
MaxBlur pool (\textsf{MBP}) \citep{zhang2019making}, 
Stylized ImageNet (\textsf{SIN}) \citep{hendrycks2019robustness}, 
AugMix (\textsf{AM}) \citep{Hendrycks2020augmix}, 
AugMix w. SIN (\textsf{AMS}) \citep{Hendrycks2020augmix}. 
We use the performance reported in \citep{Hendrycks2020augmix}. 
Again, our augmention only uses the generic texture with perturbation (the $\ma$ in our texture synthetic experiments with radius changed to $20, 25, 30, 35, 40$).
The results are reported in Table~\ref{tab:real:c} (with more details in Table~\ref{tab:real:c:app}), which shows that
our generic method outperform the current SOTA methods after a continued finetuning process with reducing learning rates. 

\paragraph{Cross-domain ImageNet-Sketch Classification}
We also compare to the methods used for cross-domain evaluation. 
We follow the set-up advocated by \citep{wang2019learning2} 
for domain-agnostic cross-domain prediction, 
which is training the model on one or multiple domains 
without domain identifiers and test the model on an unseen domain. 
We use the most challenging setup in this scenario: 
train the models with standard ImageNet training data, 
and test the model over ImageNet-Sketch data \citep{wang2019learning}, which is a collection of sketches 
following the structure ImageNet validation set. 
We compare with previous methods with reported performance on this dataset, 
such as \textsf{InfoDrop} \citep{achille2018information}, 
\textsf{HEX} \citep{wang2019learning2}, and \textsf{PAR} \citep{wang2019learning}, 
and report the performances in Table~\ref{tab:real:sketch}. 
Notice that, our data augmentation also follows the 
requirement that the characteristics of the test domain cannot 
be utilized during training. 
Thus, we only augment the samples with a generic augmentation set
($\ma$ of ``contrast'' in synthetic experiments). 
The results again support the strength of the correct usage of data augmentation.

\section{Conclusion}
\label{sec:con}

In this paper, 
we conducted a systematic inspection to study the proper regularization techniques that are provably related to the 
generalization error of a machine learning model, 
when the test distribution are allowed to be perturbed by a family of transformation functions. 
With progressively more specific assumptions, 
we identified progressively simpler methods that can bound the worst case risk. We summarize the main \textbf{take-home messages} below:

\begin{itemize}
    \item Regularizing a norm distance between the logits of the originals samples and the logits of the augmented samples enjoys several merits:
    the trained model tend to have good worst cast performance, and can learn the concept of invariance (as shown in our invariance test). 
    Although our theory suggests $\ell_1$ norm, but we recommend squared $\ell_2$ norm in practice considering the difficulties of passing the (sub)gradient of $\ell_1$ norm in backpropagation. 
    \item With the vertex assumption held (it usually requires domain knowledge to choose the vertex functions), one can use ``regularized training with vertices'' method and get good empirical performance in both accuracy and invariance, and the method is at the same complexity order of vanilla training without data augmentation. 
    When we do not have the domain knowledge (thus are not confident in the vertex assumption), 
    we recommend ``regularized worst-case augmentation'', 
    which has the best performance overall, but requires extra computations to identify the worst-case augmentated samples at each iteration. 
\end{itemize}



\bibliography{ref}
\bibliographystyle{abbrvnat}

\newpage 
\appendix

\section{Additional Assumptions}
\label{sec:app:assumption}

\begin{itemize}
\item [\textbf{A4}:] We list two classical examples here: 
\begin{itemize}[leftmargin=*]
    \item when \textbf{A4} is ``$\Theta$ is finite, $l(\cdot, \cdot)$ is a zero-one loss, samples are \textit{i.i.d}'',  $\phi(|\Theta|, n, \delta)=\sqrt{(\log(|\Theta|) + \log(1/\delta))/2n}$
    \item when \textbf{A4} is ``samples are \textit{i.i.d}'', $\phi(|\Theta|, n, \delta) = 2\mathcal{R}(\mathcal{L}) + \sqrt{(\log{1/\delta})/2n}$, where $\mathcal{R}(\mathcal{L})$ stands for Rademacher complexity and $\mathcal{L} = \{l_\theta \,|\, \theta \in \Theta \}$, where $l_\theta$ is the loss function corresponding to $\theta$. 
\end{itemize}
For more information or more concrete examples of the generic term, 
one can refer to relevant textbooks such as \citep{bousquet2003introduction}. 
\end{itemize}

\begin{itemize}
    \item[\textbf{A5}:] the worst distribution for expected risk 
equals the worst distribution for 
empirical risk, \textit{i.e.}, 
\begin{align*}
    \argmax_{\mathcal{P}'\in T(\mP, \ma)} \rpp(\wt) 
    = \argmax_{\mathcal{P'}\in T(\mP, \ma)}\wrpp(\wt) 
\end{align*}
where $T(\mP, \ma)$ is the collection of distributions created by elements in $\ma$ over samples from $\mP$. 
\end{itemize}
Assumption \textbf{A5} appears very strong, 
however, the successes of methods like adversarial training \citep{MadryMSTV18} 
suggest that, in practice, 
\textbf{A5} might be much weaker 
than it appears. 

\begin{itemize}
    \item [\textbf{A6}:] With $(\x,\y) \in (\X,\Y)$, the worst case sample in terms of maximizing cross-entropy loss and worst case sample in terms of maximizing classification error for model $\wt$ follows: 
    \begin{align}
        \forall \x, \quad \dfrac{\y^\top \fxt}{\inf_{a\in \ma}\y^\top \faxt} \geq \exp\big(\mathbb{I}(g(\fxt)\neq g(f(\x';\wt)))\big)
        \label{eq:a2}
    \end{align}
    where $\x'$ stands for the worst case sample in terms of maximizing classification error, \textit{i.e.}, 
    \begin{align*}
        \x' = \argmin_\x \y^\top g(\fxt)
    \end{align*}
    Also, 
    \begin{align}
        \forall \x, \quad \vert \inf_{a\in \ma}\y^\top \faxt \vert \geq 1
    \label{eq:assum:lipschitz}
    \end{align}
\end{itemize}

Although Assumption \textbf{A6} appears complicated, it describes simple situations that we will unveil in two scenarios: 
\begin{itemize}
    \item If $g(\fxt)=g(f(\x';\wt))$, which means either the sample is misclassified by $\wt$ or the adversary is incompetent to find a worst case transformation that alters the prediction, the RHS of Eq.~\ref{eq:a2} is 1, thus Eq.~\ref{eq:a2} always holds (because $\ma$ has the identity map as one of its elements). 
    \item If $g(\fxt)\neq g(f(\x';\wt))$, which means the adversary finds a transformation that alters the prediction. In this case, A2 intuitively states that the $\ma$ is reasonably rich and the adversary is reasonably powerful to create a gap of the probability for the correct class between the original sample and the transformed sample. The ratio is described as the ratio of the prediction confidence from the original sample over the prediction confidence from the transformed sample is greater than $e$.
\end{itemize}

We inspect Assumption \textbf{A6} by directly 
calculating the frequencies out of all the samples 
when it holds. 
Given a vanilla model (\base{}), we notice that over 74\% samples out of 50000 samples fit this assumption. 

\newpage 
\section{Proof of Theoretical Results}

\subsection{Proof of Lemma 3.1}
\textbf{Lemma.}
\textit{
With Assumptions A1, A4, and A5, with probability at least $1-\delta$, we have 
\begin{align}
    \sup_{\mathcal{P'}\in T(\mP, \ma)} 
    \rpp(\wt)  \leq 
    \dfrac{1}{n}\sum_{(\x, \y) \sim \mP}\sup_{a\in \ma}\mI(g(\faxt) \neq \y)  + \phi(|\Theta|, n, \delta)
\end{align}
}

\begin{proof}
With Assumption A5, we simply say
\begin{align*}
    \argmax_{\mP'\in T(\mP, \ma)} \rpp(\wt) 
    = \argmax_{\mP'\in T(\mP, \ma)}\wrpp(\wt) = \mP_w
\end{align*}
we can simply analyze the expected risk following the standard classical techniques since both expected risk and empirical risk are studied over distribution $\mP_w$. 

Now we only need to make sure the classical analyses (as discussed in A4) are still valid over distribution $\mP_w$:
\begin{itemize}
    \item when \textbf{A4} is ``$\Theta$ is finite, $l(\cdot, \cdot)$ is a zero-one loss, samples are \textit{i.i.d}'',  $\phi(|\Theta|, n, \delta)=\sqrt{\dfrac{\log(|\Theta|) + \log(1/\delta)}{2n}}$. 
    The proof of this result uses Hoeffding's inequality, which only requires independence of random variables. One can refer to Section 3.6 in \cite{liang2016cs229t} for the detailed proof. 
    \item when \textbf{A4} is ``samples are \textit{i.i.d}'', $\phi(|\Theta|, n, \delta) = 2\mathcal{R}(\mathcal{L}) + \sqrt{\dfrac{\log{1/\delta}}{2n}}$. 
    The proof of this result relies on McDiarmid's inequality, which also only requires independence of random variables. One can refer to Section 3.8 in \cite{liang2016cs229t} for the detailed proof. 
\end{itemize}
Assumption A1 guarantees the samples from distribution $\mP_w$ are still independent, thus the generic term holds for at least these two concrete examples, thus the claim is proved. 

\end{proof}

\subsection{Proof of Proposition 3.2}
\textbf{Proposition.}
\textit{
With A2, and $d_e(\cdot,\cdot)$ in A2 chosen to be $\ell_1$ norm, for any $a\in \ma$, we have
\begin{align}
    \sum_{i} ||\fxit-\faxit||_1 = W_1(\fxt, \faxt)
\end{align}
}
\begin{proof}
We leverage the order statistics representation of Wasserstein metric over empirical distributions (\textit{e.g.}, see Section 4 in \cite{bobkov2019one})
\begin{align*}
    W_1(\fxt, \faxt) = \inf_{\sigma}\sum_{i}||\fxit - f(a(\x_{\sigma(i)}), \wt)||_1
\end{align*}
where $\sigma$ stands for a permutation of the index, thus the infimum is taken over all possible permutations.
With Assumption A2, when $d_e(\cdot,\cdot)$ in A2 chosen to be $\ell_1$ norm, we have: 
\begin{align*}
    ||\fxit - \faxit||_1 \leq \min_{j\neq i} ||\fxit - f(a(\x_{j}), \wt)||_1
\end{align*}
Thus, the infimum is taken when $\sigma$ is the natural order of the samples, which leads to the claim. 
\end{proof}

\subsection{Proof of Theorem 3.3}
\begin{theorem*}
With Assumptions A1, A2, A4, A5, and A6, and $d_e(\cdot,\cdot)$ in A2 is $\ell_1$ norm, with probability at least $1-\delta$, the worst case generalization risk will be bounded as
\begin{align}
    \sup_{\mathcal{P'}\in T(\mP, \ma)} 
    \rpp (\wt)  \leq 
    \wrp (\wt) + 
    \sum_{i}||f(\x_i;\wt) - f(\x'_i;\wt)||_1 + 
    \phi(|\Theta|, n, \delta)
\end{align}
and 
$\x' = a(\x) $, where $ a = \argmax_{a \in \ma} \y^\top \faxt$.
\end{theorem*}

\begin{proof}

First of all, in the context of multiclass classification, where $g(f(\x, ;\theta))$ predicts a label with one-hot representation, and $\y$ is also represented with one-hot representation, we can have the empirical risk written as:
\begin{align*}
    \wrp(\x;\wt) = 1 -
    \dfrac{1}{n}\sum_{(\x, \y) \sim \mP}\y^\top g(\fxt)
\end{align*}
Thus, 
\begin{align*}
    \sup_{\mathcal{P'}\in T(\mP, \ma)} 
    \wrpp(\x;\wt) 
    & = \wrp(\x;\wt) + 
    \sup_{\mathcal{P'}\in T(\mP, \ma)} 
    \wrpp(\x;\wt) -
    \wrp(\x;\wt) \\
    & = \wrp(\x;\wt) + \dfrac{1}{n}\sup_{\mathcal{P'}\in T(\mP, \ma)} \big(\sum_{(\x, \y) \sim \mP} \y^\top g(\fxt) - \sum_{(\x, \y) \sim \mP'}\y^\top g(\fxt)\big)
\end{align*}

With A6, we can continue with: 
\begin{align*}
    \sup_{\mathcal{P'}\in T(\mP, \ma)} 
    \wrpp(\x;\wt) 
    & \leq \wrp(\x;\wt)
    + \dfrac{1}{n}\sup_{\mathcal{P'}\in T(\mP, \ma)} \big(\sum_{(\x, \y) \sim \mP} \y^\top \log(\fxt) - \sum_{(\x, \y) \sim \mP'}\y^\top \log(\fxt)\big)
\end{align*}
If we use $e(\cdot)=-\y^\top \log(\cdot)$ to replace the cross-entropy loss, we simply have:
\begin{align*}
    \sup_{\mathcal{P'}\in T(\mP, \ma)} 
    \wrpp(\x;\wt) 
    & \leq \wrp(\x;\wt)
    + \dfrac{1}{n}\sup_{\mathcal{P'}\in T(\mP, \ma)} \big(\sum_{(\x, \y) \sim \mP'} e(\fxt) - \sum_{(\x, \y) \sim \mP}e((\fxt)\big)
\end{align*}
Since $e(\cdot)$ is a Lipschitz function with constant $\leq 1$ (because of A6, Eq.\eqref{eq:assum:lipschitz}) and together with the dual representation of Wasserstein metric (See \textit{e.g.}, \cite{villani2003topics}), 
we have
\begin{align*}
    \sup_{\mathcal{P'}\in T(\mP, \ma)} 
    \wrpp(\x;\wt) 
    & \leq \wrp(\x;\wt)
    +  W_1(f(\x, \wt), f(\x', \wt))
\end{align*}
where $\x' = a(\x) $, where $ a = \argmax_{a \in \ma} \y^\top \faxt$.

Further, we can use the help of Proposition 3.2 to replace Wassertein metric with $\ell_1$ distance. 
Finally, we can conclude the proof with Assumption A5 as how we did in the proof of Lemma 3.1.    
\end{proof}

\subsection{Proof of Lemma 3.4}
\textbf{Lemma.}
\textit{
With Assumptions A1-A6, and $d_e(\cdot, \cdot)$ in A2 chosen as $\ell_1$ norm distance, 
$d_x(\cdot, \cdot)$ in A3 chosen as Wasserstein-1 metric,
assuming there is a $a'()\in \ma$ where $ \widehat{r}_{\mP_{a'}}(\wt)=\frac{1}{2}\big(\widehat{r}_{\mP_{a^+}}(\wt)+\widehat{r}_{\mP_{a^-}}(\wt)\big)$,
with probability at least $1-\delta$, we have:
\begin{align}
    \sup_{\mathcal{P'}\in T(\mP, \ma)} 
    \rpp(\wt)  \leq &
    \dfrac{1}{2}\big(\widehat{r}_{\mP_{a^+}}(\wt) + \widehat{r}_{\mP_{a^-}}(\wt)\big) + 
    \sum_{i}||f(a^+(\x_i);\wt) - f(a^-(\x');\wt)||_1 + 
    \phi(|\Theta|, n, \delta)
\end{align}
}
\begin{proof}
We can continue with 
\begin{align*}
    \sup_{\mathcal{P'}\in T(\mP, \ma)} 
    \wrpp(\x;\wt) 
    & \leq \wrp(\x;\wt)
    +  W_1(f(\x, \wt), f(\x', \wt))
\end{align*}
from the proof of Lemma 3.3. 
With the help of Assumption A3, we have:
\begin{align*}
    d_x (f(a^+(\x), \wt), f(a^-(\x), \wt)) \geq d_x (f(\x, \wt), f(\x', \wt))
\end{align*}
When $d_x(\cdot, \cdot)$ is chosen as Wasserstein-1 metric, we have: 
\begin{align*}
    \sup_{\mathcal{P'}\in T(\mP, \ma)} 
    \wrpp(\x;\wt) 
    & \leq \wrp(\x;\wt)
    +  W_1(f(a^+(\x), \wt), f(a^-(\x), \wt))
\end{align*}
Further, as the LHS is the worst case risk generated by the transformation functions within $\ma$, 
and $\wrp(\x;\wt)$ is independent of the term $W_1(f(a^+(\x), \wt), f(a^-(\x), \wt))$, 
WLOG, we can replace $\wrp(\x;\wt)$ with the risk of an arbitrary distribution generated by the transformation function in $\ma$. 
If we choose to use $ \widehat{r}_{\mP_{a'}}(\wt)=\frac{1}{2}\big(\widehat{r}_{\mP_{a^+}}(\wt)+\widehat{r}_{\mP_{a^-}}(\wt)\big)$, 
we can conclude the proof, with help from Proposition 3.2 and Assumption A5 as how we did in the proof of Lemma 3.3.  
\end{proof}

\newpage 
\section{Synthetic Results to Validate Assumptions}
\label{sec:app:assump:validation}

We test the assumptions introduced in this paper over MNIST data and rotations as the variation of the data. 

\noindent \textbf{Assumption A2:}
We first inspect Assumption A2, 
which essentially states the distance $d_e(\cdot, \cdot)$ 
is the smaller between a sample and its augmented copy (60$^{\circ}$ rotation)
than the sample and the augmented copy from any other samples. 
We take 1000 training examples and calculate the $\ell_1$ pair-wise distances
between the samples and its augmented copies, 
then we calculated the frequencies
when the A2 hold for one example. 
We repeat this for three different models, 
the vanilla model, 
the model trained with augmented data, 
and the model trained with regularized adversarial training. 
The results are shown in the Table~\ref{tab:assumption:a2} 
and suggest that, 
although the A2 does not hold in general, 
it holds for regularized adversarial training case, 
where A2 is used. 
Further, we test the assumption in a more challenging case, 
where half of the training samples 
are 15$^{\circ}$ rotations of the other half, 
thus we may expect the A2 violated for every sample. 
Finally, as A2 is essentially introduced 
to replace the empirical Wasserstein distance 
with $\ell_1$ distances of the samples and the augmented copies, 
we directly compare these metrics. 
However, as the empirical Wasserstein distance is 
forbiddingly hard to calculate (as it involves permutation statistics), we use a greedy heuristic to calculate
by iteratively picking the nearest neighbor of a sample and then remove the neighbor from the pool for the next sample. 
Our inspection suggests that, 
even in the challenging scenario, 
the paired distance is a reasonably good representative of 
Wasserstein distance for 
regularized adversarial training method. 

\begin{table}[]
\small
\begin{tabular}{c|ccc|ccc}
\hline
 & \multicolumn{3}{c}{Vanilla Scenario} & \multicolumn{3}{|c}{Challenging Scenario} \\
 & Vanilla & Augmented & Regularized & Vanilla & Augmented & Regularized \\ \hline
Frequency & 0.005 & 0.152 & 0.999 & 0.001 & 0.021 & 0.711 \\ \hline
Paired Distance & 217968.06 & 42236.75 & 1084.4 & 66058.4 & 28122.45 & 4287.31 \\
Wasserstein (greedy) & 152736.47 & 38117.77 & 1084.4 & 37156.5 & 20886.7 & 4218.53 \\
Paired/Wasserstein & 1.42 & 1.10 & 1 & 1.77 & 1.34 & 1.02 \\ \hline
\end{tabular}
\caption{Empirical results from synthetic data for Assumption A2.}
\label{tab:assumption:a2}
\end{table}

\noindent \textbf{Assumption A3:} 
Whether Assumption A3 hold will depend on the application and the domain knowledge of vertices, thus here we only discuss the general performances if we assume A3 hold. 
Conveniently, this can be shown by comparing 
the performances of \ra{} and the rest methods in the experiments reported in Section~\ref{sec:exp:sync}: 
out of six total scenarios (\{texture, rotation, contrast\} $\times$ \{MNIST, CIFAR10\}), there are four scenarios where \ra{} outperforms \vwa{}, this suggests that the domain-knowledge of vertices can actually help in most cases, although not guaranteed in every case.


\noindent \textbf{Assumption A6:}
We inspect Assumption A6 by directly 
calculating the frequencies out of all the samples 
when it holds. 
Given a vanilla model (\base{}), we notice that over 74\% samples fit this assumption. 


\newpage 

\section{Additional Details of Synthetic Experiments Setup}
\label{sec:app:synthetic}

\begin{table}[]
\small 
\centering 
\begin{tabular}{c|ccc|ccc|ccc}
\hline
 & \multicolumn{3}{c|}{Texture} & \multicolumn{3}{c|}{Rotation} & \multicolumn{3}{c}{Contrast} \\
 & C & R & I & C & R & I & C & R & I \\ \hline
\base{} & \textbf{0.7013} & 0.3219 & 0.714 & 0.7013 & 0.0871 & 0.5016 & 0.7013 & 0.2079 & 0.34 \\
\va{} & 0.6601 & 0.5949 & 0.9996 & 0.7378 & 0.4399 & 0.6168 & 0.7452 & 0.6372 & 0.4406 \\
\ra{} & 0.6571 & 0.6259 & \textbf{1} & 0.6815 & 0.5166 & 0.852 & \textbf{0.7742} & 0.6325 & \textbf{0.535} \\
\vwa{} & 0.6049 & 0.5814 & \textbf{1} & 0.714 & 0.6009 & 0.9172 & 0.7387 & \textbf{0.6708} & 0.479 \\
\rwa{} & 0.663 & \textbf{0.6358} & \textbf{1} & \textbf{0.7606} & \textbf{0.6486} & \textbf{0.9244} & 0.7489 & 0.6326 & 0.3736 \\ \hline
\end{tabular}
\caption{Results of CIFAR10 data. (``C'' stands for clean accuracy, 
``R'' stands for robustness, and ``I'' stands for invariance score): 
invariance score shows big differences while accuracy does not.}
\label{tab:main:cifar}
\end{table}

\textbf{Results Discussion} Table~\ref{tab:main:cifar} tells roughly the same story with Table~\ref{tab:main:mnist}. 
The invariance score of the worst case methods in Table~\ref{tab:main:cifar}
behave lower than we expected, 
we conjecture this is mainly because 
some elements in $\ma$ of ``contrast'' will transform the data 
into samples inherently hard to predict 
(\textit{e.g.} $a(\x)=\x/4$ will squeeze the pixel values together, 
so the images look blurry in general and hard to recognize), 
the model repeatedly identifies these case as the worst case and ignores the others. 
As a result, \rwa{} effectively degrades to 
\ra{}
yet is inferior to \ra{} because it does not have
the explicit vertex information. 
To verify the conjecture, we count how often each augmented sample 
to be considered as the worst case:
for ``texture'' and ``rotation'', each augmented sample generated by $\ma$ 
are picked up with an almost equal frequency, 
while for ``constrast'', $\x/2$ and $(1-\x)/2$ 
are identified only $10\%$-$15\%$ of the time $\x/4$ 
and $(1-\x)/4$ are identified as the worst case. 

\newpage
\section{More Synthetic Results}
\label{sec:app:more}

\begin{table}[]
\centering 
\begin{tabular}{lcccccc}
\hline
 & Worst & Clean & Vertex & All & Beyond & Invariance \\ \hline
\base{} & 0.9860 & \multicolumn{2}{c}{0.9921} & 0.9911 & 0.9463 & 0.9236 \\
\va{} & 0.9906 & \textbf{0.9928} & \textbf{0.9925} & \textbf{0.9927} & 0.9650 & 0.9876 \\
\ra{} & 0.9904 & 0.9909 & 0.9910 & 0.9909 & 0.9747 & 1 \\
\vwa{} & 0.9903 & \multicolumn{2}{c}{0.9922} & 0.9923 & 0.9696 & 0.9940 \\
\rwa{} & \textbf{0.9911} & \multicolumn{2}{c}{0.9915} & 0.9915 & \textbf{0.9773} & \textbf{1} \\ \hline
\hline 
\ral{} & 0.9897 & 0.9904 & 0.9901 & 0.9903 & 0.9728 & \textbf{1} \\
\raw{} & 0.9858 & 0.9888 & 0.9902 & 0.9893 & 0.9433 & 0.6428 \\
\rad{} & 0.9892 & 0.9921 & 0.9912 & 0.9919 & 0.9373 & 0.2588 \\
\rak{} & 0.0980 & 0.0980 & 0.0980 & 0.0980 & 0.0980 & 0.2800 \\
\ras{} & 0.9898 & 0.9917 & 0.9919 & 0.9920 & 0.9633 & 0.9928 \\
\rasl{} & 0.9904 & 0.9925 & 0.9918 & 0.9925 & 0.9672 & 0.9960 \\ \hline
\end{tabular}
\caption{More methods tested with more comprehensive metrics over MNIST on texture}
\label{tab:more:mnist:texture}
\end{table}

\begin{table}[]
\centering 
\begin{tabular}{lcccccc}
\hline
 & Worst & Clean & Vertex & All & Beyond & Invariance \\ \hline
\base{} & 0.2960 & \multicolumn{2}{c}{0.9921} & 0.7410 & 0.8914 & 0.2056 \\
\va{} & 0.9336 & 0.9884 & 0.9886 & 0.9775 & 0.8711 & 0.5628 \\
\ra{} & 0.9525 & 0.9930 & 0.9919 & 0.9829 & 0.9201 & 0.6044 \\
\vwa{} & 0.9408 & \multicolumn{2}{c}{0.9466} & 0.9827 & 0.5979 & 0.6284 \\
\rwa{} & \textbf{0.9882} & \multicolumn{2}{c}{\textbf{0.9934}} & \textbf{0.9934} & \textbf{0.9417} & \textbf{0.8856} \\ \hline \hline 
\ral{} & 0.9532 & 0.9913 & 0.9916 & 0.9824 & 0.9145 & 0.5912 \\
\raw{} & 0.9274 & 0.9882 & 0.9875 & 0.9757 & 0.8514 & 0.4600 \\
\rad{} & 0.9368 & 0.9895 & 0.989 & 0.9782 & 0.8431 & 0.4132 \\
\rak{} & 0.9424 & 0.9875 & 0.9872 & 0.9762 & 0.9194 & 0.6800 \\
\ras{} & 0.9389 & 0.9900 & 0.9901 & 0.9792 & 0.8631 & 0.6060 \\
\rasl{} & 0.9424 & 0.9913 & 0.9901 & 0.9804 & 0.8663 & 0.5864 \\ \hline
\end{tabular}
\caption{More methods tested with more comprehensive metrics over MNIST on rotation. }
\label{tab:more:mnist:rotation}
\end{table}

\begin{table}[]
\centering 
\begin{tabular}{lcccccc}
\hline
 & Worst & Clean & Vertex & All & Beyond & Invariance \\ \hline
\base{} & 0.2699 & \multicolumn{2}{c}{0.9921} & 0.6377 & 0.2988 & 0.2003 \\
\va{} & 0.9837 & 0.9922 & 0.9917 & 0.9913 & 0.6044 & 0.4153 \\
\ra{} & 0.9823 & 0.9936 & 0.9930 & 0.9911 & 0.6512 & 0.4166 \\
\vwa{} & 0.4470 & \multicolumn{2}{c}{0.5360} & 0.7515 & 0.4649 & 0.2210 \\
\rwa{} & \textbf{0.9893} & \multicolumn{2}{c}{\textbf{0.9940}} & \textbf{0.9930} & 0.4841 & \textbf{0.8786} \\ \hline \hline 
\ral{} & 0.9776 & 0.9935 & 0.9932 & 0.9902 & 0.6251 & 0.4176 \\
\raw{} & 0.7357 & 0.9867 & 0.9865 & 0.9361 & \textbf{0.6547} & 0.2960 \\
\rad{} & 0.9833 & 0.9913 & 0.9921 & 0.9909 & 0.6199 & 0.2000 \\
\rak{} & 0.9105 & 0.9894 & 0.9882 & 0.9677 & 0.6001 & 0.4153 \\
\ras{} & 0.9839 & 0.9916 & 0.9910 & 0.9906 & 0.6221 & 0.4273 \\
\rasl{} & 0.9844 & 0.9920 & 0.9918 & 0.9909 & 0.5843 & 0.4236 \\ \hline
\end{tabular}
\caption{More methods tested with more comprehensive metrics over MNIST on contrast. }
\label{tab:more:mnist:contrast}
\end{table}

\subsection{Experiment Setup}
To understand these methods, 
we introduce a more comprehensive test of these methods, 
including the five methods discussed in the main paper, 
and multiple ablation test methods, including 
\begin{itemize}[leftmargin=*]
\item \ral{}: when squared $\ell_2$ norm of \ra{} is replaced by $\ell_1$ norm. 
\item \raw{}: when the norm distance of \ra{} is replaced by Wasserstein distance, enabled by the implementation of Wasserstein GAN \cite{arjovsky2017wasserstein,gulrajani2017improved}. 
\item \rad{}: when the norm distance of \ra{} is replaced by a discriminator. Our implementation uses a one-layer neural network. 
\item \rak{}: when the norm distance of \ra{} is replaced by KL divergence. 
\item \ras{}: when the regularization of \ra{} is applied to softmax instead of logits. 
\item \rasl{}: when the regularization of \ra{} is applied to softmax instead of logits, and the squared $\ell_2$ norm is replaced by $\ell_1$ norm. This is the method suggested by pure theoretical discussion if we do not concern with the difficulties of passing gradient through backpropagation. 
\end{itemize}

And we test these methods in the three scenarios mentioned in the previous section: texture, rotation, and contrast. 
The overall test follows the same regime as the one reported in the main manuscript, with additional tests:
\begin{itemize}
    \item Vertex: average test performance on the perturbed samples with the vertex function from $\ma$. 
    Models with worst case augmentation are not tested with vertex as these models do not have the specific concept of vertex. 
    \item All: average test performance on all the samples perturbed by all the elements in $\ma$. 
    \item Beyond: To have some sense of how well the methods can perform in the setting that follows the same concept, but not considered in $\ma$, and not (intuitively) limited by the verteics of $\ma$, we also test the accuracy of the models with some transformations related to the elements in $\ma$, but not in $\ma$, To be specific: 
    \begin{itemize}
        \item Texture: $\ma_{\textnormal{beyond}} = \{a_5(), a_4()\}$. 
        \item Rotation: $\ma_{\textnormal{beyond}} = \{a_330(), a_345()\}$.
        \item Contrast: $\ma_{\textnormal{beyond}} = \{ a(\x)=\x/2+0.5, a(\x)=\x/4+0.75, a(\x)=(1-\x)/2+0.5, a(\x)=(1-\x)/4+0.75 \}$
    \end{itemize}
    We report the average test accuracy of the samples tested all the elements in $\ma_{\textnormal{beyond}}$
\end{itemize}

\subsection{Results}

\begin{table}[]
\centering 
\begin{tabular}{lcccccc}
\hline
 & Worst & Clean & Vertex & All & Beyond & Invariance \\ \hline
\base{} & 0.3219 & \multicolumn{2}{c}{0.7013} & 0.5997 & 0.3084 & 0.7140 \\
\va{} & 0.5949 & 0.6601 & 0.6394 & 0.6530 & 0.5583 & 0.9996 \\
\ra{} & 0.6259 & 0.6571 & 0.6485 & 0.6553 & 0.5826 & \textbf{1} \\
\vwa{} & 0.5814 & \multicolumn{2}{c}{0.6049} & 0.6024 & 0.5213 & \textbf{1} \\
\rwa{} & \textbf{0.6358} & \multicolumn{2}{c}{0.6630} & 0.6612 & \textbf{0.5892} & \textbf{1} \\ \hline \hline 
\ral{} & 0.6230 & 0.6609 & 0.6511 & 0.6578 & 0.5775 & \textbf{1} \\
\raw{} & 0.6140 & 0.6860 & 0.6578 & 0.6783 & 0.5801 & \textbf{1} \\
\rad{} & 0.5794 & \textbf{0.7663} & \textbf{0.6734} & \textbf{0.7288} & 0.5632 & 0.3220 \\
\rak{} & 0.5866 & 0.5873 & 0.5868 & 0.5870 & 0.5804 & \textbf{1} \\
\ras{} & 0.6197 & 0.6263 & 0.6268 & 0.6266 & 0.5831 & \textbf{1} \\
\rasl{} & 0.6319 & 0.653 & 0.6480 & 0.6516 & 0.5830 & \textbf{1} \\ \hline
\end{tabular}
\caption{More methods tested with more comprehensive metrics over CIFAR10 on texture}
\label{tab:more:cifar:texture}
\end{table}

\begin{table}[]
\centering 
\begin{tabular}{lcccccc}
\hline
 & Worst & Clean & Vertex & All & Beyond & Invariance \\ \hline
\base{} & 0.0871 & \multicolumn{2}{c}{0.7013} & 0.4061 & 0.4634 & 0.5016 \\
\va{} & 0.4399 & 0.7378 & 0.7199 & 0.6835 & \textbf{0.5096} & 0.6168 \\
\ra{} & 0.5166 & 0.6815 & 0.6741 & 0.6452 & 0.4408 & 0.8520 \\
\vwa{} & 0.6009 & \multicolumn{2}{c}{0.7140} & 0.7406 & 0.4446 & 0.9172 \\
\rwa{} & \textbf{0.6486} & \multicolumn{2}{c}{\textbf{0.7606}} & \textbf{0.7507} & 0.4614 & \textbf{0.9244} \\ \hline \hline 
\ral{} & 0.4685 & 0.7505 & 0.7290 & 0.6852 & 0.4878 & 0.6248 \\
\raw{} & 0.4228 & 0.7468 & 0.7287 & 0.6822 & 0.4753 & 0.6072 \\
\rad{} & 0.4298 & \textbf{0.7752} & 0.7456 & 0.6941 & 0.4662 & 0.2664 \\
\rak{} & 0.5848 & 0.4241 & 0.4221 & 0.4211 & 0.3946 & 0.9200 \\
\ras{} & 0.5143 & 0.7187 & 0.7175 & 0.6851 & 0.4694 & 0.8188 \\
\rasl{} & 0.4779 & 0.7341 & 0.725 & 0.6911 & 0.4944 & 0.7288 \\ \hline
\end{tabular}
\caption{More methods tested with more comprehensive metrics over CIFAR10 on rotation. }
\label{tab:more:cifar:rotation}
\end{table}

\begin{table}[]
\centering 
\begin{tabular}{lcccccc}
\hline
 & Worst & Clean & Vertex & All & Beyond & Invariance \\ \hline
\base{} & 0.2079 & \multicolumn{2}{c}{0.7013} & 0.4793 & 0.2605 & 0.3400 \\
\va{} & 0.6372 & 0.7452 & 0.7243 & 0.7365 & 0.3733 & 0.4406 \\
\ra{} & 0.6867 & \textbf{0.7742} & \textbf{0.7702} & \textbf{0.7722} & 0.5527 & 0.5350 \\
\vwa{} & 0.6708 & \multicolumn{2}{c}{0.7387} & 0.7375 & 0.5539 & 0.4790 \\
\rwa{} & 0.6326 & \multicolumn{2}{c}{0.7489} & 0.7246 & 0.4789 & 0.3736 \\ \hline \hline 
\ral{} & \textbf{0.7096} & 0.7688 & 0.7634 & 0.7666 & \textbf{0.7330} & \textbf{0.6260} \\
\raw{} & 0.6325 & 0.7442 & 0.7303 & 0.7364 & 0.4994 & 0.4396 \\
\rad{} & 0.6451 & 0.7515 & 0.7392 & 0.7479 & 0.4820 & 0.2393 \\
\rak{} & 0.1137 & 0.4515 & 0.4517 & 0.3317 & 0.2648 & 0.5026 \\
\ras{} & 0.6856 & 0.7618 & 0.7558 & 0.7609 & 0.6531 & 0.4833 \\
\rasl{} & 0.6895 & 0.7585 & 0.7533 & 0.7581 & 0.7000 & 0.4946 \\ \hline
\end{tabular}
\caption{More methods tested with more comprehensive metrics over CIFAR10 on contrast. }
\label{tab:more:cifar:contrast}
\end{table}

We report the results in Table~\ref{tab:more:mnist:texture}-\ref{tab:more:cifar:contrast}.

\paragraph{Ablation Study} 
First we consider the ablation study to validate our choice as the 
squared $\ell_2$ norm regularization, 
particularly because our choice considers both the theoretical arguments and practical arguments regarding gradients. 
In case of worst-case prediction, 
we can see the other \ra{} variants can barely outperform \ra{}, 
even not the one that our theoretical arguments directly suggest (\rasl{} or \raw{}). 
We believe this is mostly due to the challenges of
passing the gradient with $\ell_1$ norm and softmax, 
or through a classifier. 

We also test the performances of other regularizations that 
are irrelevant to our theoretical studies, but are popular choices in general (\rad{} and \rak{}). 
These methods in general perform badly, can barely match \ra{} in terms of the worst-case performance. 
Further, when some cases when \rad{} and \rak{} can outperform \ra{} in other accuracy-wise testing, these methods tend to behave terribly in invariance test, which suggests these regularizations are not effective. 
In the cases when \rad{} and \rak{} can match \ra{} in invariance test, these methods can barely compete with \ra{}. 

\paragraph{Broader Test}
We also test our methods in the broader test. 
As we can see, \rwa{} behaves the best in most of the cases. 
In three out of these six test scenarios, \rwa{} lost to three other different methods in the ``beyond'' case. 
However, we believe, in general, this is still a strong evidence to show that \rwa{} is a generally preferable method. 

Also, comparing the methods of \ra{} vs. \va{}, and \rwa{} vs. \vwa{}, we can see that regularization helps mostly in the cases of ``beyond'' in addition to ``invariance'' test. 
This result again suggests the importance of 
regularizations, as in practice, 
training phase is not always aware of all the transformation functions during test phase.

\newpage
\section{Additional Discussions for Comparisons with Advanced Methods}
\label{sec:app:real}

\begin{table}[]
\centering
{\tiny    
\hspace*{-1.5cm}\begin{tabular}{cc|ccc|cccc|cccc|cccc|c}
\hline
\multirow{2}{*}{} & \multirow{2}{*}{Clean} & \multicolumn{3}{c|}{Noise} & \multicolumn{4}{c|}{Blur} & \multicolumn{4}{c|}{Weather} & \multicolumn{4}{c|}{Digital} & \multirow{2}{*}{mCE} \\ \cline{3-17}
 &  & Gauss & Shot & Impulse & Defocus & Glass & Motion & Zoom & Snow & Frost & Fog & Bright & Contrast & Elastic & Pixel & JPEG &  \\ \hline
\base{} & 23.9 & 79 & 80 & 82 & 82 & 90 & 84 & 80 & 86 & 81 & 75 & 65 & 79 & 91 & 77 & 80 & 80.6 \\
\va{} & 23.7 & 79 & 80 & 79 & 75 & 87 & 80 & 79 & 78 & 76 & 69 & 58 &70 &86 &73 &75 & 76.3 \\
\ra{} & 23.6 & 78 & 78 & 79 & 74 & 87 & 79 & 76 & 78 & 75 & 69 & 58 & 68 & 85 & 75 & 75 & 75.6 \\
\rwa{} & 23.1 & 76 & 77 & 78 & 71 & 86 & 76 & 75 & 75 & 73 & 66 & 55 & 68 & 83 & 76 & 73 & 73.9 \\
\vwa{} & 22.4 & 61 & 63 & 63 & 68 & 75 & 65 & 66 & 70 & 69 & 64 & 56 & 55 & 70 & 61 & 63 & 64.6 \\
\textsf{SU} & 24.5 & 67 & 68 & 70 & 74 & 83 & 81 & 77 & 80 & 74 & 75 & 62 & 77 & 84 & 71 & 71 & 74.3 \\
\textsf{AA} & 22.8 & 69 & 68 & 72 & 77 & 83 & 80 & 81 & 79 & 75 & 64 & 56 & 70 & 88 & 57 & 71 & 72.7 \\
\textsf{MBP} & 23 & 73 & 74 & 76 & 74 & 86 & 78 & 77 & 77 & 72 & 63 & 56 & 68 & 86 & 71 & 71 & 73.4 \\
\textsf{SIN} & 27.2 & 69 & 70 & 70 & 77 & 84 & 76 & 82 & 74 & 75 & 69 & 65 & 69 & 80 & 64 & 77 & 73.3 \\
\textsf{AM} & 22.4 & 65 & 66 & 67 & 70 & 80 & 66 & 66 & 75 & 72 & 67 & 58 & 58 & 79 & 69 & 69 & 68.4 \\
\textsf{AMS} & 25.2 & 61 & 62 & 61 & 69 & 77 & 63 & 72 & 66 & 68 & 63 & 59 & 52 & 74 & 60 & 67 & 64.9 \\ \hline
\end{tabular}
}
\caption{Comparison to advanced models over ImageNet-C data. Performance reported (mCE) follows the standard in ImageNet-C data: mCE is the smaller the better.}
\label{tab:real:c:app}
\end{table}

\textbf{Cross-domain ImageNet-Sketch Classification}

\end{document}